\documentclass[conference,letterpaper]{IEEEtran}

\usepackage[utf8]{inputenc} 
\usepackage[T1]{fontenc}
\usepackage{amsmath, amssymb, amsfonts}
\usepackage{graphicx}
\usepackage{enumerate}
\usepackage{tabularx}
\usepackage{array}
\usepackage{dsfont}
\usepackage{bm}
\usepackage{xcolor}
\usepackage{algorithm}
\usepackage{algpseudocode}
\usepackage{multirow}
\usepackage{url}
\usepackage{cite}
\usepackage{balance}

\newtheorem{theorem}{Theorem}
\newtheorem{lemma}{Lemma}

\newtheorem{definition}{Definition}

\newtheorem{assumption}{Assumption}
\newtheorem{remark}{Remark}

\begin{document}

\title{Decentralized Conformal Novelty Detection via Quantized Model Exchange}

\author{
\IEEEauthorblockN{Kyle Loh and Yu Xiang}
\IEEEauthorblockA{Department of Mathematics and Statistics\\
                   Florida Atlantic University \\
                  Boca Raton, Florida \\
                 Email: \{kloh2023, yxiang\}@fau.edu} }

\maketitle

\begin{abstract}
This work studies decentralized novelty detection with global false discovery rate (FDR) control across heterogeneous composite null distributions, without sharing the raw data due to privacy and bandwidth considerations. We propose to extend the AdaDetect framework to decentralized settings based on the exchange of quantized surrogate models, allowing independent agents to share low-precision representations of locally learned non-conformity score functions. We prove that evaluating data against these quantized composite scores preserves conditional exchangeability, providing rigorous finite-sample guarantees for global FDR control. Empirical studies on synthetic datasets confirm our theoretical results, demonstrating that the proposed approach maintains competitive statistical power while drastically reducing the communication cost.
\end{abstract}

\section{Introduction}

\noindent In large-scale novelty detection, controlling the rate of false positives is paramount for reliable inference. The false discovery rate (FDR) provides a scalable metric for this control via the celebrated Benjamini-Hochberg (BH) procedure~\cite{benjamini1995controlling}, supported by a foundational literature extending it to correlated statistics~\cite{benjamini2001control}, Bayesian interpretations~\cite{storey2003positive, efron2001empirical}, and composite null hypotheses~\cite{benjamini2009selective}. Building on conformal prediction (e.g., see~\cite{vovk2005algorithmic,angelopoulos2024theoretical} and the references therein) and the BH procedure~\cite{benjamini1995controlling}, several recent frameworks such as~\cite{bates2023testing,marandon2024adaptive} have been developed, leveraging machine learning models to construct data-driven score functions that boost statistical power while preserving finite-sample FDR guarantees.

We focus on AdaDetect~\cite{marandon2024adaptive}, a rigorous methodology for conformal anomaly detection. This framework assumes a centralized architecture where all data can be pooled and evaluated simultaneously. In the real world, massive datasets are often decentralized across distributed networks. For example, financial institutions must collaboratively detect sophisticated fraud rings without exposing proprietary customer transaction records~\cite{dalpozzolo2015calibrating}. Similar constraints exist in multi-hospital healthcare networks~\cite{nguyen2022federated} and distributed IoT sensor arrays~\cite{xie2011anomaly}. Privacy regulations and severe communication bandwidth limitations explicitly prohibit the pooling of raw data in these environments.

On the other hand, existing works have explored decentralized FDR control mechanisms~\cite{ramdas2017qute,pournaderi2023large,pournaderi2023sample}---including those designed for adversarial settings~\cite{zhang2025distributed}---that allow independent agents to control global error rates through the cooperative exchange of summary statistics of pre-computed local test statistics (i.e., $p$-values) rather than raw data. However, they are focused on the BH procedure and do not directly address the challenge of adaptive learning of highly non-linear score functions to capture the boundaries of a global and heterogeneous composite null distribution. 
 
In this work, we make an attempt to extend the AdaDetect framework~\cite{marandon2024adaptive} to a decentralized setting under a global composite null hypothesis. When agents hold heterogeneous local distributions, isolated evaluations fail because local models lack the domain knowledge to recognize external definitions of normalcy. To enforce the  consensus required by a global composite null, we propose a decentralized framework based on \emph{model exchange} with quantization. Rather than sharing raw data or scalar thresholds, agents exchange quantized surrogate score functions along with scalar rejection counts. This approach preserves the conditional exchangeability required for rigorous FDR control and reduces total communication. Our experiments demonstrate competitive statistical power on synthetic datasets.


\section{Background}

\subsection{Problem Statement: Decentralized Novelty Detection}
Consider a network of $K \ge 2$ independent agents. Each agent $j$ holds a private local dataset consisting of true null data drawn from an \emph{unknown} underlying local distribution $P_0^{(j)}$ and novelties each drawn from \emph{unknown} and potentially \emph{different} distributions. 

The objective of decentralized novelty detection is to evaluate a distributed set of test points simultaneously and classify them as either normal or anomalous with respect to the global composite null distribution. We make an attempt to develop decentralized procedures that satisfy two conditions:
\begin{enumerate}
    \item \textbf{Global FDR control:} The expected proportion of false discoveries across the entire network must be  bounded by a target level $\alpha$. Let $V^{(j)}$ and $R^{(j)}$ denote the number of false discoveries and total rejections at agent~$j$, respectively. The procedure must guarantee:
    \begin{equation*}
        \text{FDR}_{\text{global}} = \mathbb{E} \left[ \frac{\sum_{j=1}^K V^{(j)}}{\left(\sum_{j=1}^K R^{(j)}\right) \vee 1} \right] \le \alpha,
    \end{equation*}
    where $\{a \vee b\}:=\max\{a,b\}$.
    
    \item \textbf{Privacy or communication-constraints:} Achieve this global FDR control without transmitting raw data measurements.
\end{enumerate}

Specifically, we build on a seminal framework, AdaDetect~\cite{marandon2024adaptive}, which operates by evaluating the conformity of test points against a reference set of known nulls in the simple null problem. First, it constructs conditionally exchangeable non-conformity scores. Crucially, the conditional exchangeability property guarantees that the framework can calculate empirical $p$-values, which
are then passed into the standard Benjamini-Hochberg (BH) procedure, leading to FDR control.

\subsection{Notation}\label{section: notation}
Consider a network of $K \ge 2$ independent agents. Following the positive-unlabeled (PU) classification formulation of the AdaDetect framework~\cite{marandon2024adaptive}, we formalize the local datasets. For any Agent $j$, noting that the specific values of $n, m, m_0,k$ may vary across agents, we observe:
\begin{itemize}
    \item null training samples $Y^{(j)} = (Y^{(j)}_1, \dots, Y^{(j)}_n)$ of ``typical'' measurements;
    \item test samples $X^{(j)} = (X^{(j)}_1, \dots, X^{(j)}_m)$ of ``unlabeled'' measurements;
    \item the fully ordered sequence $Z^{(j)} = (Z^{(j)}_1, \dots, Z^{(j)}_{n+m}) = (Y^{(j)}_1, \dots, Y^{(j)}_n, X^{(j)}_1, \dots, X^{(j)}_m)$.
\end{itemize}

Let $\mathcal{H}_0^{(j)}$ and $\mathcal{H}_1^{(j)}$ denote the index sets of the true nulls and true novelties within the test sample, respectively. The sequence of all true nulls for Agent $j$ is defined as $U^{(j)} = (Y^{(j)}_1, \dots, Y^{(j)}_n, X^{(j)}_i, i \in \mathcal{H}_0^{(j)})$, and the sequence of novelties is $V^{(j)} = (X^{(j)}_i, i \in \mathcal{H}_1^{(j)})$.

The null training sample $Y^{(j)} = (Y^{(j)}_1, \dots, Y^{(j)}_n)$ is split into a disjoint training set of size $k$ and a calibration set of size $\ell = n - k$.  Following the AdaDetect methodology~\cite{marandon2024adaptive}, Agent $j$ uses the training set and the unlabeled test sample to learn a local scoring function $g^{(j)}$. This function evaluates any given observation and assigns a scalar score, where higher values of $g^{(j)}(x)$ indicate a greater departure from the local null distribution $P_0^{(j)}$.

\subsubsection*{Composite Null Setting}

We define the global composite null distribution as the family of local null distributions across the network, $\mathcal{P}_0 = \{P_0^{(1)}, \dots, P_0^{(K)}\}$. 

For any given test observation $X_i^{(j)}$, we define:
\begin{itemize}
    \item $X_i^{(j)}$ is a true composite null\footnote{Equivalently, an observation $X^{(j)}_i$ is a true composite null if $X^{(j)}_i \sim P_0^{(j)}$ for some $j \in \{1, \dots, K\}$.} if $X_i^{(j)} \sim P_0^{(j)}$    \item $X_i^{(j)}$ is a true composite novelty if $X_i^{(j)} \nsim P_0^{(j)}$ for all $j \in \{1, \dots, K\}$.
\end{itemize}

\subsubsection{Data Assumptions}
To prove that our decentralized composite scores satisfy these conditions, we introduce the following formal assumptions mapping directly to the original AdaDetect framework:

\smallskip
\begin{assumption}[Exchangeability]
\label{assumption:exchangeability_invariance}
For each agent $j$:
\begin{enumerate}[(i)]
    \item The sequence of raw null observations $U^{(j)}$ is exchangeable conditionally on the alternative observations $V^{(j)}$.
    \item The local score function $g^{(j)}$ evaluates any arbitrary point $z$ identically regardless of the internal ordering of the subset $\{Z^{(j)}_{k+1}, \dots, Z^{(j)}_{n+m}\}$, provided the sequence of the first $k$ elements remains fixed.
\end{enumerate}
\end{assumption}
\smallskip

\begin{remark}
Under the PU classification framework, the conditional exchangeability of the individual non-conformity scores is implied by Assumption~\ref{assumption:exchangeability_invariance} combined with the properties of the datasets $H^{(k)}$ (see details from~\cite[Lemma 3.2]{marandon2024adaptive}).
\end{remark}
\smallskip

\begin{assumption}[Almost Surely No Ties]
\label{assumption:no_ties}
In the unquantized and quantized model exchange settings, the composite scores evaluated on the null and alternative datasets almost surely have no ties.
\end{assumption}
\smallskip

\subsection{Baselines}
\label{sec:baselines}
In decentralized novelty detection, the goal is to identify points that are anomalous to the entire network. To evaluate methods that attempt to achieve this without data pooling, we establish comparative baselines, starting from an idealized centralized setting. Throughout this work, we assume that the datasets are independent across all agents. 

\smallskip
\begin{assumption}[Independence]
\label{assumption:agent_independence}
The local datasets for all $K$ agents are mutually independent.
\end{assumption}
\smallskip

\begin{lemma}[Global FDR Bound for Independent Agents] \label{lem:global_fdr_sum}
If $K$ independent agents independently control their local FDR such that $\text{FDR}^{(j)} \le \alpha/K$ for all $j \in \{1, \dots, K\}$, then the global FDR across the disjoint union of their rejection sets is bounded by $\alpha$.
\end{lemma}
\smallskip

This follows from the simple observation that $\text{FDR}_{\text{global}}\le \sum_{j=1}^K \mathbb{E} \left[ \frac{V^{(j)}}{R^{(j)} \vee 1} \right] = \alpha$.
\smallskip

\subsubsection{Baseline 1: Centralized Oracle Setting}\label{Baseline: OracleHull}
The ideal approach to testing on composite nulls relies on Bayesian methods where one knows the prior. If the network possesses a known mixture probability distribution for the true nulls, the overall distribution can be obtained as a weighted integral, effectively reducing the composite null problem to a simple null problem (see, e.g.,~\cite{benjamini2009selective,chi2010multiple}). In this oracle setting, the $K$ agents completely pool their raw local datasets.  

\smallskip
\begin{remark}
Pooling raw data across agents requires a central server, which imposes prohibitive communication bandwidth costs and violates decentralized data privacy constraints. Also, in real-world exploratory scenarios, one almost never has access to the true prior mixture distribution; agents typically only possess empirical datasets.
\end{remark}
\smallskip

\subsubsection{Baseline 2: Zero Communication}\label{Baseline: Isolated}
A naive approach to decentralized detection is the isolated setting, where agents perform local tests at level $\alpha/K$ without any communication. This prevents the isolated approach from leveraging the full global error budget $\alpha$, causing a loss in statistical power. Furthermore, while this bounds the global error via Lemma~\ref{lem:global_fdr_sum}, it fails to enforce consensus. Because the agents remain blind to other distributions, they may accept points that are only ``locally" normal but anomalous to the rest of the network. Consequently, this approach is fundamentally incompatible with evaluating a global composite null unless all underlying distributions are identical.

\subsubsection{Baseline 3: Limited Communication (FastLSU)}\label{Baseline: AdaFastLSU}
Another baseline is to combine AdaDetect with the Fast Local Step-Up (FastLSU) algorithm~\cite{madar2016fastlsu}, described as follows. Each Agent $j$ applies AdaDetect locally: independently trains a local scoring function $g^{(j)}$ and computes empirical conformal $p$-values $\mathbf{P}^{(j)}$ for its local test sample $X^{(j)}$. The agents then execute FastLSU, which interactively shares scalar rejection counts across the network to iteratively locate the exact global $p$-value rejection threshold of a centralized BH procedure. 

Because FastLSU transmits bounded integer counts rather than continuous model weights or raw data, its communication cost per agent can be strictly quantified. During each iteration, an agent transmits its local rejection count requiring $\lceil \log_2(m + 1) \rceil$ bits, and receives the aggregated global count requiring $\lceil \log_2(K m + 1) \rceil$ bits. In the worst-case scenario, the algorithm must evaluate every unique $p$-value across the network, requiring a maximum of $K m$ iterations. Thus, the theoretical upper bound on the total communication payload per agent is $Km\left( \lceil \log_2(m + 1) \rceil + \lceil \log_2(K m + 1) \rceil \right)$ bits. In practice, however, FastLSU typically converges in very few iterations, resulting in a drastically reduced empirical payload.

\smallskip
\begin{remark}
    It should be noted that this iterative aggregation typically requires a central server. While the requirement for a dedicated central server can be structurally prohibitive in purely decentralized environments, a practical workaround is to designate one of the participating local agents to serve as the central coordinator for the algorithm.
\end{remark}
\smallskip

This approach is effective for the general distributed global FDR control problem, where the sole objective is to bound the overall error rate across a disjoint set of independent local tests. To guarantee the baseline validity of the FastLSU algorithm across heterogeneous distributions in this general setting, we formally establish that the concatenated $p$-value vector is Positive Regression Dependent on a Subset~\cite{madar2016fastlsu}.

\smallskip
\begin{definition}[Positive Regression Dependence on a Subset (PRDS)]
A random vector $\mathbf{X} = (X_1, \dots, X_m) \in \mathbb{R}^m$ is said to be Positive Regression Dependent on a Subset $I_0 \subseteq \{1, \dots, m\}$ if, for any increasing set $D \subseteq \mathbb{R}^m$ and for each index $i \in I_0$, the conditional probability $\mathbb{P}(\mathbf{X} \in D \mid X_i = x)$ is a non-decreasing function of $x$. 
\end{definition}
\smallskip

\begin{lemma}\label{lemma:global_prds}
The concatenated vector of independent $p$-values $\mathbf{P} = (\mathbf{P}^{(1)}, \dots, \mathbf{P}^{(K)})$ satisfies the global PRDS property on the combined subset of true null hypotheses $\mathcal{H}_0 = \cup_{j=1}^K \mathcal{H}_0^{(j)}$.
\end{lemma} 
\smallskip

A detailed proof for $K=2$ is provided in Appendix~\ref{sec:globalPRDSproof}. By induction, the concatenated vector $\mathbf{P}$ remains PRDS for any $K \in \mathbb{N}$ independent agents.

While Baseline 3 controls the global FDR at $\alpha$ without data sharing, it is insufficient when applied specifically to the multiple hypothesis testing problem under a composite null hypothesis. In this setting, global FDR control is not guaranteed to hold for the FastLSU framework unless the underlying null distributions are identical. 

Due to the lack of information shared, FastLSU only manages to synchronize the global rejection threshold, leaving the non-conformity scores to be evaluated by isolated local models. If the distributions are heterogeneous ($P_0^{(i)} \neq P_0^{(j)}$ for $i\neq j$), Agent~$i$'s model $g^{(i)}$ solely measures conformity to $P_0^{(i)}$. Consequently, Agent~$i$ will falsely accept test points that belong to its local distribution but lie entirely outside the composite null distribution. Because local models lack the domain knowledge to evaluate external distributions, the only scenario where local conformity guarantees global consensus is when all null distributions are identical.


\section{Methods} 

To address the multiple hypothesis global FDR control problem over heterogeneous composite nulls, we propose a decentralized framework based on data-splitting, quantized model exchange, and FastLSU. Recall that each agent has null training samples $Y^{(j)} = (Y^{(j)}_1, \dots, Y^{(j)}_n)$ and 
test samples $X^{(j)} = (X^{(j)}_1, \dots, X^{(j)}_m)$. Recall that we are working with a quite general setting, especially the novelties are allowed to have \emph{arbitrary dependency among themselves at each agent}. To handle this challenging setting, we need the following block-wide data independence for our proposed methodology. 

\smallskip
\begin{assumption}[Block-wise Data Independence]
\label{assumption:block_independence}
Each agent can split its data into $K$ \emph{independent} datasets each consisting of $1/K$ of the training samples and $1/K$ of the test samples.
\end{assumption}
\smallskip

For simplicity of presentation, we assume that both $n/K$ and $m/K$ are integers. It is worth noting that this assumption is automatically satisfied when a much stronger assumption holds, that is, when all the data samples are mutually independent (this is indeed needed for the power analysis in~\cite[Assumption~4]{marandon2024adaptive}). 

\begin{remark}
\label{remark:sparse}
The primary limitation of the proposed framework is the requirement to split data into $K$ blocks. In sparse scenarios or scenarios where one cannot invoke Assumption~\ref{assumption:block_independence}, splitting the data $K$ ways may leave the learning algorithm $\mathcal{A}$ with insufficient samples to accurately estimate the null distribution boundaries, leading to degraded statistical power. In such data-sparse scenarios, agents should use the conservative approach outlined in Appendix~\ref{sec:sparse_regime}, termed as the \emph{Conservative Model Exchange regime}, which utilizes $100\%$ of the local data but incurs a loss in power due to the $\alpha/K$ level.
\end{remark}

\subsection{Step 1: Block-wise Data Splitting}
To leverage the full global error budget $\alpha$ without pooling data, the network must rely on a decentralized thresholding procedure (detailed in Step 4). However, the theoretical validity of this procedure  requires the final evaluation $p$-values across all agents to be mutually independent. As formally established in Lemma~\ref{lemma:global_prds}, this independence is the mathematical prerequisite for concatenating local $p$-values while preserving the necessary dependency structure for global FDR control.

To guarantee this independence without sacrificing network consensus, agents split their data. Each agent $j \in \{1, \dots, K\}$ divides its local PU dataset $H^{(j)}$ into $K$ disjoint blocks. One block is reserved for Agent $j$'s own local calibration and final test evaluation. The remaining $K-1$ blocks are designated for training surrogate models for the other agents in the network. 

\subsection{Step 2: Composite Score Maximization with Quantization}
\label{sec:q_model}

In a centralized setting, testing a global composite null hypothesis often relies on calculating the worst-case significance level across all constituent null distributions. This is equivalent to taking the minimum $p$-value across all local tests; applying the BH procedure to these minimum $p$-values is a well-established method proven to control the global FDR~\cite{benjamini2009selective}. However, aggregating these exact global $p$-values traditionally requires the pooling of raw data across all agents.

To achieve this consensus without transmitting raw data, we propose exchanging the learned score functions themselves. Because higher non-conformity scores map directly to lower $p$-values, taking the maximum non-conformity score across all shared models serves as a privacy-preserving surrogate for evaluating the minimum $p$-value. 

Formally, let $H^{(a)}$ denote the block sent from agent $a$ to agent $j$; to simplify notation, we write $H^{(a)}$ instead of $H^{(a\to j)}$ and in the proof we explicitly label the blocks. Each agent $j$ evaluates its local observations against the set of all shared datasets to construct a composite score for any arbitrary $z \in \mathcal{Z}$:
\begin{equation}\label{non-quantized-score}
    S^{(j)}(z) = \max_{a \in \{1,\dots,K\}} \left\{ g^{(a)}(z; H^{(a)}) \right\}.
\end{equation}
Accordingly, we define the non-conformity scores for the evaluated observations as $S^{(j)}_i = S^{(j)}(Z_i^{(j)})$. The composite null scores and novelty scores are explicitly denoted as $S_{U,i}^{(j)}$ and $S_{V,m}^{(j)}$, respectively.

By evaluating data against all functions simultaneously, the agent flags anomalies that are considered anomalous by any member of the network. For each test point $X_t^{(j)} := Z_{n+t}^{(j)}$, the empirical $p$-value is calculated by comparing its score against the $\ell$ local calibration scores~\cite{marandon2024adaptive}:
\begin{equation}\label{empirical-p-value}
    p_t^{(j)} = \frac{1}{\ell + 1} \left( 1 + \sum_{i=k+1}^n \mathds{1}_{\left\{ S_i^{(j)} \ge S_{n+t}^{(j)} \right\}} \right), \quad t \in \{1, \dots, m\}.
\end{equation}

\noindent{\bf Quantization.} To alleviate the communication burden of transmitting high-resolution datasets across the network and to preserve data privacy, agents exchange a quantized surrogate of their learned model parameters. Let $\mathcal{A}: \mathcal{H} \to \Theta$ be a fixed, deterministic learning algorithm mapping the dataset $H^{(a )}$ to a set of continuous model parameters $\theta^{(a )}$. Let $q: \Theta \to \hat{\Theta}$ be a fixed, deterministic, quantization function mapping these parameters to a discrete, lower-resolution space. 

Instead of transmitting $H^{(a )}$, i.e., weights $\theta^{(a )}$, Agent $a$ computes and transmits the quantized parameters $\hat{\theta}^{(a )} = q(\mathcal{A}(H^{(a )}))$. The receiving Agent $j$ then utilizes a surrogate score function parameterized by these low-resolution weights:
\begin{equation}\label{surrogatescorefunction}
    F^{(a )}(\cdot) := f\left(\cdot \; ; q(\mathcal{A}(H^{(a )}))\right) = f\left(\cdot \; ; \hat{\theta}^{(a )}\right)
\end{equation}

To construct the quantized composite score for any point $z$, Agent $j$ retains its original score function for its local evaluations, and compares it against the low-fidelity surrogate functions received from all remote agents $a \neq j$:
\begin{equation}\label{quantizedcompositescore}
    \tilde{S}^{(j)}(z) = \max \left\{ g^{(j )}\left(z; H^{(j )}\right), \max_{a \neq j} F^{(a )}\left(z\right) \right\}.
\end{equation}
Consequently, the quantized null and novelty scores are calculated as $\tilde{S}_{U,i}^{(j)} := \tilde{S}^{(j)}(U_i^{(j)})$ and $\tilde{S}_{V,m}^{(j)} := \tilde{S}^{(j)}(V_m^{(j)})$ respectively. 

\subsection{Step 3: Global FDR Control via FastLSU}
Because the models Agent $j$ received were trained on data blocks disjoint from its evaluation block, the resulting empirical $p$-values across the network are mutually independent. 

The agents execute the decentralized Fast Local Step-Up (FastLSU) algorithm~\cite{madar2016fastlsu}. Instead of pooling $p$-values, the network interactively shares scalar rejection counts across the network to iteratively locate the exact global rejection threshold of a centralized BH procedure, guaranteeing global FDR control at target level $\alpha$.

Crucially, the validity of this global FDR guarantee requires that the underlying empirical $p$-values satisfy the PRDS property. As established in Section~\ref{Baseline: AdaFastLSU} (Lemma~\ref{lemma:global_prds}), the concatenated global vector of independent, locally PRDS $p$-values rigorously satisfies this requirement.

\subsection{Theoretical Guarantees}
To establish that our proposed framework guarantees global FDR control at level $\alpha$, we must prove two sequential properties. First, we must demonstrate that aggregating independent local tests via the FastLSU protocol preserves the PRDS property required by the BH procedure. As established in Section~\ref{sec:baselines}, the concatenated global vector of independent, locally PRDS $p$-values satisfies this requirement (Lemma~\ref{lemma:global_prds}). Second, we must prove that our method of computing composite non-conformity scores achieves conditional exchangeability, which ensures that the resulting empirical $p$-values locally satisfy the PRDS property.

The foundational methodology of our approach relies on evaluating the conformity of local test points against a reference set of known nulls. In the simple null problem, AdaDetect achieves valid FDR control by ensuring that the constructed non-conformity scores are conditionally exchangeable. By constructing a PU dataset to train a scoring function, the framework calculates valid empirical $p$-values, which are then passed into the standard BH procedure. The fundamental guarantee of this framework is formalized as follows.

\smallskip
\begin{theorem}[Theorem 3.3 from~\cite{marandon2024adaptive}]
\label{thm:adadetect_prds}
If the sequence of null scores is exchangeable conditionally on the alternative scores, and if there are almost surely no ties among the scores, then the resulting empirical $p$-values satisfy the PRDS property on the subset of true null hypotheses, and thus the FDR control.
\end{theorem}
\smallskip

We are now ready to present our main theoretical results for the composite null setting, under Assumption~\ref{assumption:agent_independence}, Assumption~\ref{assumption:block_independence}, Assumption~\ref{assumption:exchangeability_invariance}, and Assumption~\ref{assumption:no_ties}. 

\smallskip
\begin{lemma}[Conditional Exchangeability of Composite Scores]
\label{lem:composite_exchangeability}
Under the stated multi-agent setting in~\eqref{non-quantized-score}, the sequence of Agent $j$'s null scores $S_U^{(j)}$ is exchangeable conditional on the alternative scores.
\end{lemma}
\smallskip

\begin{IEEEproof}
The proof for the two-agent case ($K=2$) is provided in Appendix~\ref{sec:compositePRDSproof}. Due to the mutual independence of the local datasets (Assumption~\ref{assumption:agent_independence}), this proof extends to arbitrary $K \in \mathbb{N}.$
\end{IEEEproof}
\smallskip

\begin{theorem}[Global FDR Control: Non-Quantized Model Exchange]
\label{theorem:local_fdr}
Across all agents, the concatenation of all local empirical $p$-values, as defined in~\eqref{empirical-p-value}, satisfies the PRDS property on $\mathcal{H}_0$, guaranteeing global FDR control at the target level $\alpha$ via the FastLSU procedure.
\end{theorem}
\smallskip

\begin{IEEEproof}
By Lemma~\ref{lem:composite_exchangeability}, the composite scores are conditionally exchangeable. This guarantees the PRDS property of the local empirical $p$-values via Theorem~\ref{thm:adadetect_prds}. Because the exchanged unquantized models are trained on data blocks disjoint from each agent's block, the resulting local $p$-value vectors are mutually independent. By Lemma~\ref{lemma:global_prds}, concatenating these independent, locally PRDS $p$-value vectors yields a global vector that satisfies the global PRDS property on $\mathcal{H}_0$. This fulfills the theoretical requirement for the FastLSU procedure to bound the global FDR at level $\alpha$.
\end{IEEEproof}

\smallskip
\begin{theorem}[Global FDR Control: Quantized Model Exchange]
\label{theorem:quantized_fdr}
Given $(\mathcal{A},q)$ defined in Section~\ref{sec:q_model}, the composite scores defined in~\eqref{quantizedcompositescore} satisfy conditional exchangeability, guaranteeing global FDR control at level $\alpha$.
\end{theorem}
\smallskip

Due to space limitations, we defer the proof to Appendix~\ref{sec:quantized_proof}.


\section{Experiments}
\label{sec:experiments}

To empirically validate the theoretical guarantees of our  framework, we conduct a series of experiments comparing the Zero Communication baseline (Baseline 2 or B2), the Limited Communication baseline (Baseline 3 or B3), and our proposed model exchange (ME) framework. 

\subsection{Experimental Setup}
We simulate decentralized networks of $K =3$ agents, each training a local Random Forest (RF) classifier acting as the scoring function. The global target FDR is set to $\alpha = 0.1$. Performance is evaluated across 100 independent trials.

\subsubsection{Simulated Data}
We generate a $d=20$ dimensional Gaussian dataset with a sparse anomaly signature. The global dataset consists of $4000$ total samples, comprising $3000$ training observations and $1000$ testing observations. The ratio of true nulls to total observations is fixed at $0.9$ ($\pi_0 = 0.9$). This global dataset is evenly split across the $K=3$ agents. The local null distributions for the agents are arranged in a centered equilateral triangle in $\mathbb{R}^d$. Let $u$ and $v$ be orthogonal unit vectors where $u$ is supported on the first 10 coordinates ($u_i = 1/\sqrt{10}$ for $1 \le i \le 10$) and $v$ is supported on the last 10 coordinates ($v_i = 1/\sqrt{10}$ for $11 \le i \le 20$). The local nulls follow $X_{\text{null}}^{(k)} \sim \mathcal{N}(\mu_k, I_d)$, with centroids placed at distance $\delta$ from the origin: $\mu_1 = \delta u$, $\mu_2 = \delta(-u/2 + v\sqrt{3}/2)$, and $\mu_3 = \delta(-u/2 - v\sqrt{3}/2)$. We sweep the shift parameter $\delta \in \{0, 0.5, 1.0, 2.0, 3.0, 4.0\}$. The novelties remain globally stationary and follow $X_{\text{alt}} \sim \mathcal{N}(\mu_{\text{alt}}, I_d)$, featuring a sparse mean shift on the first 5 coordinates where $\mu_{\text{alt}, i} = \sqrt{2 \log(d)}$, and $0$ elsewhere.

\subsection{Synthetic Network Performance}
Tables \ref{tab:shift_sweep_power} and \ref{tab:shift_sweep_fdr} detail performance across increasing distribution shifts. Both the B2 and B3 baselines rely strictly on independent local evaluations. While increasing $\delta$ generally caused a decline in statistical power in B2 and B3 (e.g., from 0.901 down to 0.748 for B2, and from 0.949 down to 0.849 for B3), the ME framework maintained near-perfect power (above $0.99$ for $\delta \ge 3.0$) and controlled global FDR. 

\vspace{-1em}
\begin{table}[ht]
\centering
\caption{Global FDR: Shift Sweep (Synthetic, $K=3$, $\alpha=0.1$)}
\label{tab:shift_sweep_fdr}
\footnotesize
\setlength{\tabcolsep}{5pt}
\begin{tabular}{c | c c c }
\hline
\textbf{Shift ($\delta$)} & \textbf{B2 FDR} & \textbf{B3 FDR} & \textbf{ME FDR} \\
\hline
0.0 & 0.027 $\pm$ 0.01 & 0.086 $\pm$ 0.02 & 0.081 $\pm$ 0.03 \\
0.5 & 0.029 $\pm$ 0.01 & 0.088 $\pm$ 0.02 & 0.084 $\pm$ 0.03 \\
1.0 & 0.028 $\pm$ 0.01 & 0.091 $\pm$ 0.02 & 0.081 $\pm$ 0.03 \\
2.0 & 0.027 $\pm$ 0.01 & 0.087 $\pm$ 0.02 & 0.082 $\pm$ 0.03 \\
3.0 & 0.028 $\pm$ 0.01 & 0.088 $\pm$ 0.02 & 0.078 $\pm$ 0.03 \\
4.0 & 0.031 $\pm$ 0.01 & 0.090 $\pm$ 0.02 & 0.074 $\pm$ 0.03 \\
\hline
\end{tabular}
\end{table}

\vspace{-1em}
\begin{table}[ht]
\centering
\caption{Statistical Power: Shift Sweep (Synthetic, $K=3$, $\alpha=0.1$)}
\label{tab:shift_sweep_power}
\footnotesize
\setlength{\tabcolsep}{5pt}
\begin{tabular}{c | c c c }
\hline
\textbf{Shift ($\delta$)} & \textbf{B2 Pwr} & \textbf{B3 Pwr} & \textbf{ME Pwr} \\
\hline
0.0 & 0.901 $\pm$ 0.03 & 0.949 $\pm$ 0.02 & 0.966 $\pm$ 0.03 \\
0.5 & 0.894 $\pm$ 0.03 & 0.945 $\pm$ 0.02 & 0.971 $\pm$ 0.02 \\
1.0 & 0.881 $\pm$ 0.04 & 0.940 $\pm$ 0.02 & 0.982 $\pm$ 0.02 \\
2.0 & 0.835 $\pm$ 0.05 & 0.906 $\pm$ 0.03 & 0.991 $\pm$ 0.01 \\
3.0 & 0.763 $\pm$ 0.06 & 0.862 $\pm$ 0.03 & 0.997 $\pm$ 0.01 \\
4.0 & 0.748 $\pm$ 0.06 & 0.849 $\pm$ 0.04 & 0.999 $\pm$ 0.00 \\
\hline
\end{tabular}
\end{table}

\subsection{Communication and Quantization}
Table \ref{tab:quantization_sweep} illustrates the impact of quantization at a fixed shift, contrasting the proposed framework against the baselines. The empirical results validate Theorem~\ref{theorem:quantized_fdr}: quantization of the ME framework preserves conditional exchangeability and bounds global FDR. We observe that our ME framework outperforms Baselines 2 and 3 across all compressions. Furthermore, the 1-bit compression reduces the communication payload by nearly $20\times$ compared to the unquantized ME model (from $1576.9$ kb to $79.4$ kb) with negligible loss in statistical power.

\vspace{-1em}
\begin{table}[ht]
\centering
\caption{Quantization Sweep ($\delta=2.0$)}
\label{tab:quantization_sweep}
\footnotesize
\begin{tabular}{l c c c}
\hline
\textbf{Precision} & \textbf{Global FDR} & \textbf{Global Power} & \textbf{Comm. (kb)} \\
\hline
B2 & 0.029 $\pm$ 0.01 & 0.834 $\pm$ 0.06 & 0.0 $\pm$ 0.00 \\
B3 & 0.090 $\pm$ 0.02 & 0.908 $\pm$ 0.03 & 0.1 $\pm$ 0.01 \\ \cline{1-4}
ME Unquantized & 0.082 $\pm$ 0.03 & 0.990 $\pm$ 0.01 & 1576.9 $\pm$ 37.6 \\
ME 6-bit       & 0.080 $\pm$ 0.03 & 0.990 $\pm$ 0.01 & 199.0 $\pm$ 4.8 \\
ME 4-bit       & 0.078 $\pm$ 0.03 & 0.991 $\pm$ 0.01 & 150.3 $\pm$ 3.5 \\
ME 2-bit       & 0.084 $\pm$ 0.03 & 0.991 $\pm$ 0.01 & 103.3 $\pm$ 2.2 \\
ME 1-bit       & 0.087 $\pm$ 0.04 & 0.991 $\pm$ 0.01 & 79.4 $\pm$ 1.9 \\
\hline
\end{tabular}
\end{table}

\section{Acknowledgment}
This work was supported in part by the National Science Foundation
under Grant CCF-2611415. 
\balance
\newpage
\bibliographystyle{IEEEtran}
\bibliography{references}

\onecolumn
\appendices
\numberwithin{equation}{section} 

\section{Mathematical Definitions}
\label{sec:definitions}

To rigorously establish the theoretical guarantees in the subsequent proofs, we first formalize several foundational concepts regarding vector properties and dependency structures~\cite{benjamini2001control}.

\begin{definition}[Vector Inequality]
For two vectors $\mathbf{x}, \mathbf{y} \in \mathbb{R}^m$, we say that $\mathbf{y}$ is larger than or equal to $\mathbf{x}$, denoted as $\mathbf{x} \le \mathbf{y}$, if the inequality holds element-wise: $x_i \le y_i$ for all $i \in \{1, \dots, m\}$.
\end{definition}

\begin{definition}[Increasing Set]
A set $D \subseteq \mathbb{R}^m$ is called an \textit{increasing set} if, for any $\mathbf{x} \in D$ and any $\mathbf{y} \in \mathbb{R}^m$, the condition $\mathbf{x} \le \mathbf{y}$ implies that $\mathbf{y} \in D$.
\end{definition}

\begin{definition}[Non-Decreasing Function]
A function $f: \mathbb{R} \to \mathbb{R}$ is \textit{non-decreasing} if, for any $x, y \in \mathbb{R}$, $x \le y$ implies $f(x) \le f(y)$.
\end{definition}

\section{Proof of Lemma \ref{lemma:global_prds}}
\label{sec:globalPRDSproof}

\begin{IEEEproof}
To guarantee the validity of the FastLSU algorithm across independent agents, we establish that a concatenated vector of independent, locally PRDS $p$-values satisfies the global PRDS property.

Let $\mathbf{X} \in \mathbb{R}^{m_1}$ and $\mathbf{Y} \in \mathbb{R}^{m_2}$ be two independent random column vectors representing the $p$-values of two distinct agents. Let $I_X$ and $I_Y$ denote their respective subsets of true null hypothesis indices. 

We assume the following:
\begin{itemize}
    \item $\mathbf{X} \perp \!\!\! \perp \mathbf{Y}$.
    \item $\mathbf{X}$ is PRDS on $I_X$.
    \item $\mathbf{Y}$ is PRDS on $I_Y$.
\end{itemize}

We want to show that the concatenated global vector\footnote{For notational convenience throughout the remainder of this proof, we write the column vector $[\mathbf{X}^\top, \mathbf{Y}^\top]^\top$ simply as $(\mathbf{X}, \mathbf{Y})$.} $\mathbf{Z} = [\mathbf{X}^\top, \mathbf{Y}^\top]^\top \in \mathbb{R}^{m_1 + m_2}$ is PRDS on $I_Z = I_X \cup I_Y$.
Without loss of generality, assume the selected true null index belongs to the first block ($i \in I_X$). Thus, $Z_i = X_i$. By the tower property, we express the conditional probability as follows:
\begin{equation}
\mathbb{P}(\mathbf{Z} \in D \mid Z_i = x) = \mathbb{P}((\mathbf{X}, \mathbf{Y}) \in D \mid X_i = x)  = \mathbb{E}_{\mathbf{Y} \mid X_i=x} \left[ \mathbb{P} ( (\mathbf{X}, \mathbf{Y}) \in D \mid X_i = x, \mathbf{Y}) \right]
\end{equation}

Also, ($\mathbf{X} \perp \!\!\! \perp \mathbf{Y}$) implies that $(\mathbf{Y}\mid X_i=x)  \sim \mathbf{Y}$. Therefore, 
\begin{equation}
\mathbb{E}_{\mathbf{Y} \mid X_i=x} \left[ \mathbb{P} ( (\mathbf{X}, \mathbf{Y}) \in D \mid X_i = x, \mathbf{Y}) \right] = \mathbb{E}_\mathbf{Y} \left[\mathbb{P} ((\mathbf{X}, \mathbf{Y}) \in D \mid X_i = x, \mathbf{Y}) \right].
\end{equation}

To evaluate this expectation as an integral over the marginal probability measure $dP_{\mathbf{Y}}(\mathbf{y})$, we evaluate the inner probability at $\mathbf{Y} = \mathbf{y}$. Thus, the event $(\mathbf{X}, \mathbf{y}) \in D$ is equivalent to $\mathbf{X} \in D_{\mathbf{y}}$, where $D_{\mathbf{y}} = \{ \mathbf{x} \in \mathbb{R}^{m_1} : (\mathbf{x}, \mathbf{y}) \in D \}$ is the cross-section of $D$ at $\mathbf{y}$. 

Because ($\mathbf{X} \perp \!\!\! \perp \mathbf{Y}$), the conditional probability of $\mathbf{X}$ falling into the deterministic set $D_{\mathbf{y}}$ does not depend on the specific realization of $\mathbf{Y}$ beyond defining $D_{\mathbf{y}}$. Thus, the conditioning on $\mathbf{Y} = \mathbf{y}$ drops out of the inner probability:
\begin{equation}
\mathbb{P}((\mathbf{X}, \mathbf{Y}) \in D \mid X_i = x, \mathbf{Y} = \mathbf{y}) = \mathbb{P}(\mathbf{X} \in D_{\mathbf{y}} \mid X_i = x, \mathbf{Y} = \mathbf{y}) = \mathbb{P}(\mathbf{X} \in D_{\mathbf{y}} \mid X_i = x).
\end{equation}

Integrating this simplified probability over the distribution of $\mathbf{Y}$ yields:
\begin{equation}
\mathbb{E}_\mathbf{Y} \left[\mathbb{P} ((\mathbf{X}, \mathbf{Y}) \in D \mid X_i = x, \mathbf{Y}) \right] = \int_{\mathbb{R}^{m_2}} \mathbb{P}(\mathbf{X} \in D_{\mathbf{y}} \mid X_i = x) dP_{\mathbf{Y}}(\mathbf{y}).
\end{equation}

Let $\mathbf{x} \in D_{\mathbf{y}}$ and $\mathbf{x}' \ge \mathbf{x}$. The concatenated vector satisfies:
\begin{equation}
(\mathbf{x}', \mathbf{y}) \ge (\mathbf{x}, \mathbf{y})
\end{equation}
Since $D$ is an increasing set and $(\mathbf{x}, \mathbf{y}) \in D$, it follows that $(\mathbf{x}', \mathbf{y}) \in D$ and $\mathbf{x}' \in D_{\mathbf{y}}$. Thus,  $D_{\mathbf{y}}$ is increasing. Because $D_{\mathbf{y}}$ is an increasing set and $i \in I_X$, for any fixed $\mathbf{y}$, the $\mathbb{P}(\mathbf{X} \in D_{\mathbf{y}} \mid X_i = x)$ is non-decreasing in $x$ due to $\mathbf{X}$ being PRDS in $D_\mathbf{y}$.

For $x_1 \le x_2$, integrating over the non-negative probability measure $dP_\mathbf{Y}(\mathbf{y})$ yields:
\begin{equation}
\int_{\mathbb{R}^{m_2}} \mathbb{P}(\mathbf{X} \in D_{\mathbf{y}} \mid X_i = x_1) dP_{\mathbf{Y}}(\mathbf{y}) \le \int_{\mathbb{R}^{m_2}} \mathbb{P}(\mathbf{X} \in D_{\mathbf{y}} \mid X_i = x_2) dP_{\mathbf{Y}}(\mathbf{y}).
\end{equation}
 
Hence, $\mathbb{P}(\mathbf{Z} \in D \mid Z_i = x)$ is non-decreasing in $x$ for $ i \in I_X$. By symmetry, the same holds for $ i \in I_Y$. Thus, $\mathbf{Z}$ is PRDS on $I_Z$. 
\end{IEEEproof}

\section{Proof of Lemma \ref{lem:composite_exchangeability}}
\label{sec:compositePRDSproof}

\begin{IEEEproof}
We recall the positive-unlabeled (PU) classification data structures defined in Section \ref{section: notation}. In the context of block-wise data splitting, let $U^{(1)}$ and $V^{(1)}$ denote the complete set of true nulls and novelties within Agent 1's held-out evaluation block. Similarly, let $U^{(2)}$ and $V^{(2)}$ denote the true nulls and novelties within the training block that Agent 2 has designated specifically for Agent 1. Let:
\begin{align*}
    H^{(1 )} &= h^{(1 )}(U^{(1)}, V^{(1)}) = \left(Y_{\text{train}}^{(1 )}, \{Y_{\text{cal}}^{(1 )} \cup X_0^{(1 )}, V^{(1)}\}\right); \\
    H^{(2 )} &= h^{(2 )}(U^{(2)}, V^{(2)}) =  \left(Y_{\text{train}}^{(2 )}, \{Y_{\text{cal}}^{(2 )} \cup X_0^{(2 )}, V^{(2)}\}\right); \\
    S^{(1)}_{U, i} &= \max \left\{ g^{(1 )}(U^{(1)}_i; H^{(1 )}), g^{(2 )}(U^{(1)}_i; H^{(2 )}) \right\}; \\
    S^{(1)}_{V, j} &= \max \left\{ g^{(1 )}(V^{(1)}_j; H^{(1 )}), g^{(2 )}(V^{(1)}_j; H^{(2 )}) \right\}.
\end{align*}

While the first argument $Y_{\text{train}}^{(k )}$ represents a fixed sequence of known labels, the second argument in $H^{(1 )}$ and $H^{(2 )}$ is formulated as an unordered set, which is permutation invariant.  Thus, $H^{(1 )}$ and $H^{(2 )}$ are deterministic functions of $(U^{(1)}, V^{(1)})$ and $(U^{(2)}, V^{(2)})$ respectively, and are invariant to any permutations acting on the unlabeled PU mixture.

Let $\pi$ be an arbitrary permutation acting on $U^{(1)}$ such that the indices from $U^{(1)}_{\text{train}}$ are fixed. By Assumption \ref{assumption:exchangeability_invariance}, $U^{(1)} \mid V^{(1)} \sim  U^{(1)}_\pi \mid V^{(1)}$, implying $(U^{(1)}, V^{(1)}) \sim  (U^{(1)}_\pi, V^{(1)})$. By Assumption \ref{assumption:agent_independence}, the datasets from Agent 1 and Agent 2 are mutually independent, which yields:
\begin{equation}
    \left(U^{(1)}, V^{(1)}, U^{(2)}, V^{(2)}\right) \sim \left(U^{(1)}_\pi, V^{(1)}, U^{(2)}, V^{(2)}\right).
\end{equation}

Because $h^{(1 )}$ is invariant to $\pi$ and $H^{(2 )}$ depends entirely on Agent 2's independent data block, both $H^{(1 )}$ and $H^{(2 )}$ are invariant to $\pi$. Hence:
\begin{equation}
    \left(U^{(1)}, V^{(1)}, U^{(2)}, V^{(2)}, H^{(1 )}, H^{(2 )}\right) \sim  \left(U^{(1)}_\pi, V^{(1)}, U^{(2)}, V^{(2)}, H^{(1 )}, H^{(2 )}\right).
\end{equation}

Marginalizing out $U^{(2)}$ and $V^{(2)}$ yields:
\begin{equation}
    \left(U^{(1)}, V^{(1)}, H^{(1 )}, H^{(2 )}\right) \sim  \left(U^{(1)}_\pi, V^{(1)}, H^{(1 )}, H^{(2 )}\right).
\end{equation}

In other words, by conditioning on the remaining random variables:
\begin{equation}
 \left(U^{(1)} \mid V^{(1)}, H^{(1 )}, H^{(2 )}\right) \sim  \left(U^{(1)}_\pi \mid V^{(1)}, H^{(1 )}, H^{(2 )}\right).
\end{equation}

Evaluating the deterministic, element-wise maximum score functions $S_U^{(1)}$ and $S_V^{(1)}$ on these conditionally exchangeable sets preserves the symmetry:
\begin{equation}
\left(S^{(1)}_{U, 1}, \dots, S^{(1)}_{U, |U^{(1)}|}\right) \mid \left(V^{(1)}, H^{(1 )}, H^{(2 )}\right) \sim  \left(S^{(1)}_{U, \pi(1)}, \dots, S^{(1)}_{U, \pi(|U^{(1)}|)}\right) \mid \left(V^{(1)}, H^{(1 )}, H^{(2 )}\right).
\end{equation}

Because the novelty scores $S_V^{(1)}$ are deterministic evaluations over the fixed novelties $V^{(1)}$ and parameters $H^{(1 )}, H^{(2 )}$, conditioning on the full set $(V^{(1)}, V^{(2)}, H^{(1 )}, H^{(2 )})$ is a stronger condition that implies conditioning on $S_V^{(1)}$. Therefore:
\begin{equation}
\left(S^{(1)}_{U, 1}, \dots, S^{(1)}_{U, |U^{(1)}|}\right) \mid S_V^{(1)} \sim  \left(S^{(1)}_{U, \pi(1)}, \dots, S^{(1)}_{U, \pi(|U^{(1)}|)}\right) \mid S_V^{(1)},
\end{equation}
satisfying the conditional exchangeability requirement.
\end{IEEEproof}

\section{Proof of Theorem \ref{theorem:quantized_fdr}}
\label{sec:quantized_proof}

\begin{IEEEproof}
Without loss of generality, we prove the PRDS property for Agent 1's evaluation block in a $K$-agent network. Let $\mathbf{F} = \{F^{(k )}\}_{k=2}^K$ denote the set of surrogate score functions constructed from the quantized weights transmitted to Agent 1 by all other agents. Let $U^{(1)}$ and $V^{(1)}$ represent the data within Agent 1's held-out evaluation block. 

By Assumption~\ref{assumption:agent_independence} and the block-wise data splitting, Agent 1's evaluation block is mutually independent of all the other data blocks used to train the received models. Thus, Agent 1's evaluation data is independent of all the other training datasets:
\begin{equation}
(U^{(1)}, V^{(1)}) \perp\!\!\!\perp \{H^{(k )}\}_{k=2}^K 
\end{equation}

Because the training algorithm $\mathcal{A}$ and the weight quantization function $q$ are both fixed, deterministic, and measurable, their composition $q \circ \mathcal{A}$ is also a measurable transformation. Therefore, the set of functions $\mathbf{F}$ is a measurable transformation of the independent datasets. Independence is preserved under measurable transformations:
\begin{equation}
(U^{(1)}, V^{(1)}) \perp\!\!\!\perp \mathbf{F} 
\end{equation}

Because $\mathbf{F}$ is independent of Agent 1's evaluation data, it is constant with respect to any permutation $\pi$ on $U^{(1)}$. Therefore, conditioning on $\mathbf{F}$ preserves the symmetry established in Lemma~\ref{lem:composite_exchangeability}:
\begin{equation}
 \left(U^{(1)}, \{U^{(k )}\}_{k=2}^K \mid V^{(1)}, \{V^{(k )}\}_{k=2}^K, H^{(1 )}, \mathbf{F}\right) \sim  \left(U^{(1)}_\pi, \{U^{(k )}\}_{k=2}^K \mid V^{(1)}, \{V^{(k )}\}_{k=2}^K, H^{(1 )}, \mathbf{F}\right).
\end{equation}

Marginalizing out the other block observations $\{U^{(k )}, V^{(k )}\}_{k=2}^K$, we apply the deterministic element-wise maximum using Agent 1's local function and the remote low-fidelity surrogate functions to define the quantized composite scores $\tilde{S}^{(1)}$:
\begin{equation}
\tilde{S}^{(1)}_{U, i} = \max \left\{ g^{(1 )}\left(U^{(1)}_i; H^{(1 )}\right), \max_{k \neq 1} F^{(k )}\left(U^{(1)}_i\right) \right\} 
\end{equation}

Because this maximum operation relies solely on the local evaluation PU dataset $H^{(1 )}$ and the independent fixed functions in $\mathbf{F}$, the exchangeability of the resulting composite scores is preserved:
\begin{equation}
\left(\tilde{S}^{(1)}_{U, 1}, \dots, \tilde{S}^{(1)}_{U, |U^{(1)}|}\right) \mid \tilde{S}_V^{(1)} \sim  \left(\tilde{S}^{(1)}_{U, \pi(1)}, \dots, \tilde{S}^{(1)}_{U, \pi(|U^{(1)}|)}\right) \mid \tilde{S}_V^{(1)}.
\end{equation}

By Theorem~\ref{thm:adadetect_prds}, this conditional exchangeability guarantees the PRDS property of Agent 1's empirical $p$-values. By symmetry, this holds for all $j \in \{1, \dots, K\}$. Because the data blocks across all agents are mutually independent, the concatenated global vector of $p$-values satisfies the global PRDS property (Lemma~\ref{lemma:global_prds}), allowing the FastLSU procedure to guarantee global FDR control at level $\alpha$.
\end{IEEEproof}

\section{Conservative Model Exchange for Sparse Data}
\label{sec:sparse_regime}

When local datasets are too sparse to support block-wise splitting without severely degrading model accuracy, agents must utilize the Conservative Model Exchange Regime. 

\subsection{Methods}
In this approach, Agent $j$ utilizes 100\% of its local PU dataset $H^{(j)}$ to train a single model $g^{(j)}$. It transmits this identical model to all $K-1$ other agents. Agent $j$ then evaluates its test points using the composite max score $\tilde{S}^{(j)}(z) = \max \{g^{(j)}(z), \max_{a \neq j} F^{(a)}(z)\}$. 

Because the other models were trained on data that overlaps with the sets used to calculate $p$-values, independence is lost. Consequently, the network cannot utilize FastLSU and must control global FDR by applying the BH procedure locally at level of $\alpha/K$. 

To prove that this local approach bounds the global FDR, we must establish that the unsplit composite scores remain conditionally exchangeable. For any Agent $k$, the complete set of true nulls is $U^{(k)}$ and the true test novelties are $V^{(k)}$. Let $H^{(1)} = h^{(1)}(U^{(1)}, V^{(1)})$ and $H^{(2)} = h^{(2)}(U^{(2)}, V^{(2)})$ be the unsplit local datasets for Agent 1 and Agent 2. The local score functions are trained on these sets, which are invariant to any permutations acting on the unlabeled PU mixture. The rest of the proof follows from that of Lemma~\ref{lem:composite_exchangeability}.  

\begin{lemma}[Conditional Exchangeability of Conservative Composite Scores]
\label{lem:sparse_composite_exchangeability}
Under the unsplit Conservative Model Exchange setting, the sequence of Agent $j$'s null scores $S_U^{(j)}$ is exchangeable conditional on the alternative scores.
\end{lemma}

\begin{theorem}[Global FDR Control: Conservative Model Exchange]
\label{theorem:sparse_fdr}
Applying the BH procedure locally at level $\alpha/K$ on the empirical $p$-values generated by the unsplit composite scores rigorously bounds the global FDR at level $\alpha$. 
\end{theorem}

\begin{IEEEproof}
By Lemma~\ref{lem:sparse_composite_exchangeability}, the unsplit composite scores are conditionally exchangeable. This guarantees the PRDS property of the resulting local empirical $p$-values via Theorem~\ref{thm:adadetect_prds}. Applying the BH procedure on these PRDS $p$-values rigorously guarantees local FDR control at level $\alpha/K$ for each agent. By Lemma~\ref{lem:global_fdr_sum}, controlling local FDR at level $\alpha/K$ across all $K$ independent agents bounds the global FDR of the entire network at the target level $\alpha$.
\end{IEEEproof}

\subsection{Empirical Validation}
While this local $\alpha/K$ level preserves FDR, it sacrifices statistical power compared to the  non-sparse regime. Tables \ref{tab:sparse_gaussian_shift} and \ref{tab:sparse_quantization_sweep} detail the performance of the Conservative Model Exchange regime on the synthetic data setting outlined in Section \ref{sec:experiments}.

\begin{table}[ht]
\centering
\caption{Conservative Model Exchange: Simulated Dataset Shift ($\alpha=0.1$)}
\label{tab:sparse_gaussian_shift}
\footnotesize
\setlength{\tabcolsep}{3pt}
\begin{tabular}{l c c c}
\hline
\textbf{Shift ($\delta$)} & \textbf{Global FDR} & \textbf{Global Power} & \textbf{Communication (kb)} \\
\hline
$0$   & $0.030 \pm 0.01$ & $0.924 \pm 0.03$ & $15303 \pm 806$ \\
$0.5$ & $0.029 \pm 0.01$ & $0.925 \pm 0.03$ & $15349 \pm 600$ \\
$1.0$ & $0.030 \pm 0.02$ & $0.927 \pm 0.02$ & $15317 \pm 630$ \\
$2.0$ & $0.031 \pm 0.01$ & $0.890 \pm 0.03$ & $15666 \pm 777$ \\
$3.0$ & $0.028 \pm 0.01$ & $0.747 \pm 0.07$ & $16072 \pm 716$ \\
$4.0$ & $0.028 \pm 0.01$ & $0.667 \pm 0.01$ & $16036 \pm 491$ \\
\hline
\end{tabular}
\end{table}

\begin{table}[ht]
\centering
\caption{Conservative Model Exchange: Quantization Sweep ($\alpha=0.1, \delta=2.0$)}
\label{tab:sparse_quantization_sweep}
\footnotesize
\setlength{\tabcolsep}{3pt}
\begin{tabular}{l c c c}
\hline
\textbf{Precision} & \textbf{Global FDR} & \textbf{Global Power} & \textbf{Communication (kb)} \\
\hline
Unquantized & $0.031 \pm 0.01$ & $0.890 \pm 0.03$ & $15666 \pm 777$ \\
6-bit       & $0.032 \pm 0.01$ & $0.891 \pm 0.03$ & $1966 \pm 98$ \\
4-bit       & $0.031 \pm 0.01$ & $0.894 \pm 0.03$ & $1494 \pm 74$ \\
2-bit       & $0.032 \pm 0.01$ & $0.880 \pm 0.04$ & $1021 \pm 51$ \\
1-bit       & $0.031 \pm 0.01$ & $0.845 \pm 0.05$ & $785 \pm 39$ \\
\hline
\end{tabular}
\end{table}

\end{document}